\documentclass[11pt, notitlepage]{article}   

\usepackage[title]{appendix}

\usepackage{amsmath,amsthm,amsfonts,hyperref}   

\usepackage{bbm} 
\usepackage{amscd, array}
\usepackage{amsfonts}
\usepackage{amssymb}
\usepackage{color}      
\usepackage{epsfig}
\usepackage{graphicx}           
\usepackage{algorithm}
\usepackage[noend]{algpseudocode}

\usepackage{tikz}
\usetikzlibrary{shapes}
\usetikzlibrary{matrix}
\usetikzlibrary{positioning}
\usetikzlibrary{fit}

\tikzstyle{res}=[circle,thick,minimum size=4mm,draw=black,fill=red,inner sep=1pt]
\tikzstyle{non-res}=[circle,thick,minimum size=4mm,draw=black,inner sep=1pt]
\tikzstyle{light-res}=[circle,thick,minimum size=4mm,draw=black,fill=red!40,inner sep=1pt]
\tikzstyle{blue}=[circle,thick,minimum size=4mm,draw=black,fill=blue!20,inner sep=1pt]

\newtheorem{theorem}{Theorem}[section]

\newcommand{\optagst}{\textnormal{opt}_{\textnormal{ag,strong}}}
\newcommand{\optagwe}{\textnormal{opt}_{\textnormal{ag,weak}}}

\newcommand{\optbs}{\textnormal{opt}_{\textnormal{bs}}}
\newcommand{\optambr}{\textnormal{opt}_{\textnormal{amb,r}}}

\newcommand{\optstd}{\textnormal{opt}_{\textnormal{std}}}
\theoremstyle{remark}

\theoremstyle{theorem}

\newtheorem{problem}[theorem]{Problem}

\newtheorem{cor}[theorem]{Corollary}

\newtheorem{definition}[theorem]{Definition}
\newtheorem{lem}[theorem]{Lemma}
\newtheorem{thm}[theorem]{Theorem}

\newcommand{\optastd}{\textnormal{opt}_{\textnormal{cap,std}}}
\newcommand{\optabs}{\textnormal{opt}_{\textnormal{cap,bs}}}
\newcommand{\optaambr}{\textnormal{opt}_{\textnormal{cap,amb,r}}}

\newcommand{\cartr}{\textnormal{CART}_{\textnormal{r}}}

\newcommand{\relu}{\textnormal{ReLU}}
\newcommand{\A}{\textnormal{U}}
\newcommand{\Le}{\textnormal{L}}
\newcommand{\We}{\textnormal{W}}

\def\finf{\mathop{{\rm I}\kern -.27 em {\rm F}}\nolimits}

\usepackage[colorinlistoftodos,textsize=tiny]{todonotes}
\newcommand{\Comments}{1}
\newcommand{\mynote}[2]{\ifnum\Comments=1\textcolor{#1}{#2}\fi}
\newcommand{\mytodo}[2]{\ifnum\Comments=1%
  \todo[linecolor=#1!80!black,backgroundcolor=#1,bordercolor=#1!80!black]{#2}\fi}

\begin{document}

\title{Mistake-bounded online learning with operation caps}
\author{Jesse Geneson, Meien Li, and Linus Tang}

\maketitle

\begin{abstract}
    We investigate the mistake-bound model of online learning with caps on the number of arithmetic operations per round. We prove general bounds on the minimum number of arithmetic operations per round that are necessary to learn an arbitrary family of functions with finitely many mistakes. We solve a problem on agnostic mistake-bounded online learning with bandit feedback from (Filmus et al, 2024) and (Geneson \& Tang, 2024). We also extend this result to the setting of operation caps.
\end{abstract}

\section{Introduction}

With the recent surge in development and use of artificial intelligence (AI), governments around the world have passed regulations to address AI-related risks. For example, a 2023 executive order in the United States \cite{exec_order}, which was later rescinded \cite{trump}, required that cloud service providers must report customers who use more than $10^{26}$ floating point or integer operations in a single run. One of the results of this particular mandate would be for the federal government to have increased oversight for training runs of large language models.

The focus of this paper is to understand how many arithmetic operations per round are necessary to learn an arbitrary family of functions in an online setting. For a given family of functions $F$, the number of arithmetic operations per round that are necessary to learn $F$ provides a measure of complexity for $F$. 

\subsection{Online learning}

In online reinforcement learning, the learner makes a prediction for each input as it is received. After the learner makes their prediction, they receive feedback. For example, the feedback may be the correct answer for the prediction, or the feedback may be only whether the answer was correct. As the learner receives more feedback, they use the feedback to make their future answers more accurate. A standard example is price prediction.

For another example, consider a situation where a learner wants to make predictions about the weather in a certain location with no historical data. Given the last day of atmospheric conditions at the location as input, the learner would like to predict whether there will be any precipitation during the current day. Their answer will be yes or no. After they give their answer, the learner observes whether there is any precipitation during the current day, which tells them whether or not they got the right answer. 

When the learner starts the learning process, their answers are not very accurate, since there is no historical data on the atmospheric conditions and precipitation at the location. As the learner makes more guesses and receives more feedback, their predictions become more accurate. When the learner makes a mistake, it means that the input was difficult for the learner to predict based on the past inputs. However, feedback after mistakes is especially useful, since it means that similar inputs will be easier for the learner to predict in the future.

\subsection{The mistake-bound model}

In order to describe online learning more formally, we start by defining the mistake-bound model of online learning, which was introduced in \cite{angluin, littlestone}. This model can be considered a game between a learner and an adversary. There is a family $F$ of functions with domain $X$ and codomain $Y$. At the beginning of the game, the adversary picks a hidden function $f \in F$ that the learner tries to learn. In particular, in each round of the game, the adversary provides the learner with an input $x \in X$. The learner makes a guess for the output $f(x)$, and they receive some reinforcement from the adversary. 

In the \textit{standard} version of the model, also called \textit{strong reinforcement}, the adversary tells the learner the correct value of $f(x)$ after the learner's guess. In the \textit{bandit} version of the model, also called \textit{weak reinforcement}, the adversary only tells the learner whether or not they got the correct answer. The learner's goal is to make as few mistakes as possible, while the adversary's goal is to force as many mistakes as possible. Define $\optstd(F)$ to be the maximum (worst-case) number of mistakes that the learner makes during the standard learning process with family $F$, assuming that both the learner and adversary play optimally. Similarly, define $\optbs(F)$ to be the maximum (worst-case) number of mistakes that the learner makes during the bandit learning process with family $F$, assuming that both the learner and adversary play optimally. Clearly $\optstd(F) = \optbs(F)$ for any family $F$ of functions with codomain of size $2$. 

Mistake-bounded online learning is applicable to a number of sequential prediction problems including weather forecasting \cite{nms}, ad click prediction \cite{flz}, load forecasting \cite{wwlj}, and financial prediction \cite{scsnkg}. These problems are examples of online learning with standard feedback, i.e., strong reinforcement. Passcode prediction \cite{cfc, yyy} is an example of an online learning problem with bandit feedback, i.e., weak reinforcement. Given observations of a user, such as footage of their movements as they use a device, the problem is to predict their passcode. The reinforcement is weak, since you do not learn the correct passcode if you guess the wrong one. If you guess the correct passcode, then you get confirmation since you are able to log in.

Long \cite{long} proved that \[\optbs(F) \le (1+o(1)) (k \ln{k}) \optstd(F)\] for any family $F$ of functions with codomain of size $k$, where the $o(1)$ is with respect to $k$. This bound sharpened an earlier result of Auer and Long \cite{auer}. Long also investigated what is the maximum possible value of $\optbs(F)$ when $\optstd(F)$ is fixed and $F$ has codomain of size $k$. When $\optstd(F) = 1$, Long \cite{long} observed that the maximum possible value of $\optbs(F)$ is $k-1$. On the other hand, Long also showed that the maximum possible value of $\optbs(F)$ is $\Theta(k \ln{k})$ when $\optstd(F) \ge 3$, where the constants in the bound depend on $\optstd(F)$. The case when $\optstd(F) = 2$ was left as an open problem, which Geneson and Tang \cite{gt24} resolved by showing that the maximum possible value of $\optbs(F)$ is $\Theta(k \ln{k})$ when $\optstd(F) = 2$.

In addition to online learning of classifiers, there has also been some work \cite{gz23, kl_smooth, long_smooth} on online learning of functions with certain smoothness properties. 

\subsection{Agnostic learning}

Besides the standard and bandit online learning scenarios, there are a number of other scenarios that have been investigated for online learning of classifiers. For example, in agnostic learning, the adversary is allowed to choose a hidden function that is not an element of $F$, but there must exist some element of $F$ which differs from the hidden function on at most some bounded number of inputs \cite{auer, filmus, gt24}. In this paper, we consider two versions of agnostic learning, one with strong reinforcement and the other with weak reinforcement. In particular, define $\optagst(F, \eta)$ to be the worst-case number of mistakes that the learner makes during the strong agnostic learning process with family $F$, assuming that both learner and adversary play optimally, and there exists some $f \in F$ which differs from the hidden function on at most $\eta$ inputs. Moreover, define $\optagwe(F, \eta)$ to be the worst-case number of mistakes that the learner makes during the weak agnostic learning process with family $F$, assuming that both learner and adversary play optimally, and there exists some $f \in F$ which differs from the hidden function on at most $\eta$ inputs. Note that $\optagwe(F, \eta) = \optagst(F, \eta)$ for any family $F$ of functions with codomain of size $2$.

Agnostic online learning is often applicable to real world sequential prediction problems. For example, suppose that you want to train a neural network for ad click prediction. The architecture is chosen beforehand, so the number of layers, the number of neurons per layer, and the type of activation functions are all fixed. The only things that can be adjusted are the weights and the biases. If you consider the family of all neural networks with the given architecture, none of the neural networks in the family are a perfect fit for the ad click data, but some are a better fit than others, so the problem is to find a neural network in the family with the best fit. This is an agnostic online learning problem, since the goal is not to find an element of the family that perfectly fits all the data, but rather to find an element of the family that fits the data as well as possible. 

Auer and Long \cite{auer} proved that \[\optagst(F, \eta) \le 4.82(\optstd(F)+\eta)\] for families of functions $F$ with codomain of size $2$, and Cesa-Bianchi et al \cite{cesa} improved this bound to \[\optagst(F, \eta) \le 4.4035(\optstd(F)+\eta).\] More recently, Filmus et al \cite{filmus0} proved that \[\optagst(F, \eta) \le 2\eta+O(\optstd(F)+\sqrt{\optstd(F)\eta})\] for families of functions $F$ with codomain of size $2$. More generally, Geneson and Tang \cite{gt24} and Filmus et al \cite{filmus} both independently showed that \[\optagwe(F, \eta) \le k \ln{k} (1+o(1)) (\optstd(F)+\eta)\] for families of functions $F$ with codomain of size $k$, where the $o(1)$ is with respect to $k$. Both groups also showed that there exist families of functions $F$ with codomain of size $k$ such that \[\optagwe(F, \eta) = \Omega((k \ln{k})\optstd(F)+k \eta).\] Here, we improve the upper bound to show that the lower bound from \cite{filmus,gt24} is sharp up to a constant factor, i.e., we show that \[\optagwe(F,\eta) = O((k \ln{k})\optstd(F) + k \eta)\] for all families of functions $F$ with codomain of size $k$. We also show two different upper bounds for $\optagst(F, \eta)$, each of which is sharper in different cases. Specifically, we prove that \[\optagst(F,\eta) = O(\optstd(F) + \eta\ln{k})\] and \[\optagst(F,\eta)\leq (\optstd(F)+1)(\eta+1)-1.\] Note that the first upper bound is sharper when $\optstd(F)$ is much greater than $\ln{k}$, while the second upper bound is sharper when $\optstd(F)$ is much less than $\ln{k}$. 

\subsection{Online learning with delays}

Online learning has also been investigated when there is a delay between the learner's guess and reinforcement, so that the learner has to guess the output for multiple inputs before they receive reinforcement from the adversary \cite{auer, feng, gt24}. In the $r$-delayed ambiguous reinforcement model, the adversary gives the learner $r$ inputs before giving any reinforcement, and the learner must guess the corresponding output for each input before receiving the next input. The reinforcement is weak in this model, i.e., the adversary only says YES or NO at the end of the round for each output. Define $\optambr(F)$ to be the worst-case number of errors that the learner makes in the $r$-delayed ambiguous reinforcement model with the family $F$ if both the learner and the adversary play optimally. 

Passcode prediction with multifactor authentication is an example of online learning with delayed ambiguous reinforcement. Indeed, suppose that in order to log in to a secure account, a user must enter multiple passcodes from multiple devices. Given footage of the user's movements as they use the devices, the problem is to predict all the passcodes correctly. If you get all the passcodes correct, then you get confirmation, since you are able to log in. However, if any passcode is wrong, then you do not learn the correct passcodes. 

If $F$ is a family of functions with domain $X$ and codomain $Y$, let $\cartr(F)$ denote the family of functions $g: X^r \rightarrow Y^r$ such that there exists some $f \in F$ for which $g(x_1, \dots, x_r) = (f(x_1), \dots, f(x_r))$ for all $(x_1, \dots, x_r) \in X^r$. One motivation for investigating $\cartr(F)$ is that AI models usually train on batches of data points instead of one data point at a time, to improve training throughput.

Note that the $r$-delayed ambiguous reinforcement model with the family $F$ differs slightly from the bandit model with the family $\cartr(F)$. In the $r$-delayed ambiguous reinforcement model with the family $F$, the learner only receives the next input after giving their answer for the current input. In the bandit model with the family $\cartr(F)$, the learner receives all $r$ inputs at the same time. Since the learner receives more information in the bandit model with the family $\cartr(F)$, it is obvious that $\optbs(\cartr(F)) \le \optambr(F)$ for every family $F$. 

In the context of passcode prediction, the $r$-delayed ambiguous reinforcement model represents a scenario where the learner has to enter the passcodes in a specific order, and the learner only receives the footage of the user for each device after entering the guess for the previous passcode. The bandit model with the family $\cartr(F)$ represents a scenario where the learner gets all of the footage of the user for all the devices before they have to enter any guesses for the passcodes.

Despite the small difference between the two models, the worst-case error can be much greater for the $r$-delayed ambiguous reinforcement model. In particular, Feng et al \cite{feng} proved that the maximum possible value of $\frac{\optambr(F)}{\optbs(\cartr(F))}$ over all families of functions $F$ of size at least $2$ is $2^{r(1 \pm o(1))}$. For every $M > 2r$, Feng et al \cite{feng} also found a family of functions $F$ with codomain of size $k$ and $\optstd(F) = M$ such that \[\optbs(\cartr(F)) \ge (1-o(1))(k^r \ln{k})(\optstd(F)-2r).\] On the other hand, Geneson and Tang \cite{gt24} proved a general upper bound of \[\optbs(\cartr(F)) \le \optambr(F) \le (1+o(1))(k^r \ln{k})\optstd(F),\] matching the lower bound up to the leading term. 

\subsection{Learning with operation caps}

In the standard learning scenario, the learner is allowed to perform a finite but unlimited amount of computation in each round of the learning process. In particular, there is no bound on the number of arithmetic operations that the learner may use in each round. The 2023 executive order on the development and use of artificial intelligence \cite{exec_order} required cloud service providers to report customers who use more than $10^{26}$ floating point or integer operations in a single run. Moon et al \cite{mvgb} investigated practical methods for training large language models on cloud service providers while staying under the reporting threshold, as well as methods for cloud service providers to detect training runs on their platforms.

Given the threshold from the executive order, we introduce a new variant of the mistake-bound model which restricts the number of operations that the learner can use in each round. This represents a customer trying to learn a hidden function on a cloud service provider, while staying under the reporting threshold from the executive order. 

\begin{definition}
    For every natural number $a$, let $\optastd(F, a)$ be the worst-case number of mistakes that the learner makes in the standard learning scenario if the learner is only allowed to use at most $a$ binary arithmetic operations per round and both the learner and the adversary play optimally. We require for each round that the learner must use all arithmetic operations before answering, and they cannot use any arithmetic operations after answering. 
\end{definition}

Similarly, define $\optabs(F, a)$ to be the worst-case number of mistakes that the learner makes in the bandit learning scenario if the learner is only allowed to use at most $a$ binary arithmetic operations per round and both the learner and the adversary play optimally. It is immediately clear that $\optstd(F) \le \optastd(F, a)$ and $\optbs(F) \le \optabs(F, a)$ for all natural numbers $a$. 

We also define $\text{opt}_{\text{cap,ag,strong}}(F,\eta,a)$, $\text{opt}_{\text{cap,ag,weak}}(F,\eta,a)$, and $\optaambr(F,a)$ analogously for agnostic learning with strong reinforcement, agnostic learning with weak reinforcement, and delayed ambiguous reinforcement respectively.

\begin{definition}
    Let $\A_{F}$ be the minimum value such that for all $a \ge \A_{F}$, we have $\optstd(F) = \optastd(F, a)$. If there is no such $a$, then we say that $\A_{F} = \infty$.
\end{definition}

An analogous parameter can be defined for each of the other learning scenarios, but we focus mostly on the standard learning scenario in this paper. In addition to determining the minimum number of arithmetic operations per round that is necessary to obtain optimal worst-case performance for the learner in terms of number of mistakes, another natural problem is to determine the minimum number of arithmetic operations per round that is necessary to learn the family of functions with a finite number of mistakes. 

\begin{definition}
    Let $\Le_{F}$ be the minimum value of $a$ such that $\optastd(F, a) < \infty$. If there is no such $a$, then we say that $\Le_{F} = \infty$.
\end{definition} 

Clearly, for all families of functions $F$, we have $\Le_{F} \le \A_{F}$. As with $\A_{F}$, a parameter analogous to $\Le_{F}$ can be defined for each of the other learning scenarios. 

\subsection{Model of computation for online learning}

In defining models of online learning with caps on arithmetic operations, it is important to be precise about the model of computation for the learner, since the specifics of the model will determine the number of arithmetic operations used by the learning algorithm in each round of the learning process. In our learning scenario, we assume that the learner can choose a learning algorithm for each round, based on the previous inputs and outputs, such that the learning algorithm for each round can be described by a directed acyclic graph (DAG). The learner may change their DAG after the end of each round, before the next round starts. This model of computation is based on Model 2 in the technical report \cite{demmel}.

The DAG for a given round $t$ has input nodes, which contain each of the entries of the input vector for round $t$, as well as output nodes, which contain each of the entries of the learner's answer for round $t$. Between the input nodes and output nodes, there are computational nodes with arithmetic operations for addition, subtraction, multiplication, and division, as well as unary operations for the identity function, constant functions, $\lfloor x \rfloor$, $\lceil x \rceil$, $|x|$, $\sqrt{x}$, $e^x$, $\ln{x}$, and piecewise unary functions. For the piecewise unary functions, if the input is $x$, then we allow each piece to be obtained from one of the unary operations on $x$ (constant, identity, $\lfloor x \rfloor$, $\lceil x \rceil$, $|x|$, $\sqrt{x}$, $e^x$, $\ln{x}$) or a binary arithmetic operation applied to $x$ and some constant. Each node in the directed graph receives inputs, which may be constants, inputs to the algorithm, or the output of another node. Each node returns an output, which is the same as its input when the node is an input node. 

In the model of online learning with computational restrictions that we use in this paper, we cap the total number of \emph{binary} arithmetic operations in each round of the learning process. Specifically, unary arithmetic operations such as floor functions and ceiling functions do not count toward the cap. 

\subsection{Our results}

In Section~\ref{s:basic}, we prove some basic results about mistake-bounded online learning with caps on the number of arithmetic operations. Before discussing these results, we state a definition.

\begin{definition}
    For any function $f$ with domain $X$, let $\We(f)$ denote the infimum over all directed graphs $A$ for computing $f$ of the supremum over all $x \in X$ of the number of arithmetic operations that $A$ uses to compute $f(x)$, i.e., \[W(f) = \inf_A \sup_{x \in X} \#[f(x),A],\] where $\#[f(x),A]$ denotes the number of arithmetic operations that $A$ uses to compute $f(x)$. Let $\We(F)$ denote the supremum of $\We(f)$ over all $f \in F$.
\end{definition}

We show for every family of functions $F$ that $\Le_{F} \ge \We(F)$. In other words, for any algorithm $A$ that eventually learns $F$ with $\We(F) < \infty$ in the standard learning scenario, there exists a choice of hidden function and inputs for which $A$ must use at least $\We(F)$ arithmetic operations in some rounds. When $\We(F) = \infty$, the result says that for any algorithm $A$ that eventually learns $F$ in the standard learning scenario and for any positive integer $N$, there exist choices of hidden function and inputs for which $A$ must use at least $N$ arithmetic operations in some rounds. Given the inequality, it is natural to characterize the families $F$ for which $\Le_{F} = \We(F)$. While this problem seems difficult in general, we make some partial progress by exhibiting some families of functions $F$ for which $\Le_{F} > \We(F)$, as well as identifying many families $F$ for which $\Le_{F} = \We(F)$. 

In Section~\ref{s:bounds}, we determine a sufficient criterion for $\Le_{F} = \We(F)$, and we apply this criterion to several families of functions. For example, we prove that this criterion holds for the family of linear transformations from $\mathbb{R}^n$ to $\mathbb{R}$, as well as any finite family of functions. We also show that this criterion holds for several families of neural networks in Section~\ref{sec:1layer}. Online learning of neural networks was previously investigated in \cite{cmlh23, daniely25, rst15, whl24}.

In Section~\ref{s:variants} we compare the worst-case numbers of mistakes for different learning scenarios with caps on the number of arithmetic operations. In the case of agnostic learning without caps on operations, we solve a problem from \cite{filmus,gt24}, and then we use the same idea for the case of agnostic learning with caps. In Section~\ref{s:ordering}, we introduce a new technique for obtaining mistake bounds in various learning scenarios based on partial orderings of families of functions. We apply this technique to construct families of functions $F$ with $k$ outputs for which $\optstd(F) = O(n)$, $\optbs(F)=\Omega(nk\log k)$, and $\optbs(\cartr(F))=\Omega(k^n)$ for $r \ge n$. Finally, we discuss some future research directions in Section~\ref{s:conclusion}.

\section{Basic results about online learning with operation caps}\label{s:basic}

We start by finding a family of functions for which $\Le_{F} < \A_{F}$. Then, we modify the argument to find arbitrarily large gaps between $\Le_{F}$ and $\A_{F}$.

\begin{theorem}\label{thm:lvsa}
There exists a family $F$ of functions with $\Le_{F} < \A_{F}$. 
\end{theorem}
\begin{proof}
Let $F$ be the family of functions with domain $\mathbb R$ and codomain $\{0,1\}$ which return $1$ on two distinct inputs $x$ and $y$ which satisfy $xy=1$ and return $0$ on all other inputs. The learner can learn $F$ without arithmetic operations with a mistake bound of $2$ by ignoring the $xy=1$ condition, and guessing $0$ on every new input until the adversary has said the answer was wrong twice, at which point the hidden function has been determined. The learner can get the mistake bound to $1$ by using the $xy=1$ condition, but this requires an arithmetic operation. Thus, we have $\Le_{F}=0$ and $\A_{F}=1$, as desired.
\end{proof}

\begin{cor}
    For the family of functions $F$ in Theorem~\ref{thm:lvsa}, we have $\optastd(F, 0) = 2$ and $\optastd(F, a) = 1$ for all $a \ge 1$.
\end{cor}

Next, we generalize Theorem~\ref{thm:lvsa} by finding families of functions with arbitrary gaps between $\Le_{F}$ and $\A_{F}$.

\begin{theorem} \label{thm_la_gap} For every integer $r > 0$, there exists a family $F_r$ of functions with $\Le_{\text{F}_r} = 0$ and $\A_{\text{F}_r} = r$. \end{theorem} 

\begin{proof} Consider the family of functions $F_r$ which take $r$-tuples of positive real numbers as input and output either $0$ or $1$. Each function $f \in F_r$ is defined by $2r$ parameters $a_1, \dots, a_r, b_1, \dots, b_r$ with $a_1b_1 = \dots = a_r b_r = 1$, where $f(x_1, \dots, x_r) = 1$ if $(x_1,\dots,x_r) = (a_1,\dots,a_r)$ or $(x_1,\dots,x_r) = (b_1,\dots,b_r)$ and $f(x_1, \dots, x_r) = 0$ otherwise. 

The learner can learn $F_r$ without using arithmetic operations for a mistake bound of $2$ by guessing $0$ for all new $r$-tuples until the adversary has said the answer was wrong twice, at which point the function has been determined. On the other hand, the learner can get a mistake bound of $1$ by guessing $0$ until the adversary says they are wrong on the input $(z_1, \dots, z_r)$, and then checking for the subsequent input $(q_1, \dots, q_r)$ whether $z_i q_i =1$ for each $i = 1, \dots, r$. However, this requires $r$ arithmetic operations, one for each coordinate. Thus, we have $\Le_{\text{F}_r} = 0$ and $\A_{\text{F}_r} = r$. \end{proof}

\begin{cor}
    For the family of functions $F_r$ in Theorem~\ref{thm_la_gap} with $r > 0$, we have $\optastd(F_r, a) = 2$ for all $a$ with $0 \le a \le r-1$ and $\optastd(F, a) = 1$ for all $a \ge r$.
\end{cor}

In the next result, we determine a general lower bound on $\Le_{F}$ for all families $F$. 

\begin{thm}\label{triv_lower}
    For any family of functions $F$, we have $\Le_{F} \ge \We(F)$.
\end{thm}

\begin{proof}
We split into two cases, depending on whether or not $\We(F)$ is finite. 

For the first case, suppose that $\We(F) < \infty$. Let $f \in F$ be a function for which $\We(f) = \We(F)$. For contradiction, suppose that $\Le_{F} < \We(F)$. Let $m = \optastd(F,\Le(F))$. Then, there exists some algorithm $A$ for learning $F$ with at most $m$ mistakes which uses fewer than $\We(F)$ arithmetic operations per round. Let $r$ be the maximum possible number of mistakes that $A$ makes on any sequence of inputs if the hidden function is $f$. By definition, we have $r \le m$. 

There exists a sequence $S$ of inputs $x_1,x_2,\dots$ for which $A$ makes exactly $r$ mistakes on $S$ when the hidden function is $f$. Let $x_t$ be the input on which $A$ makes the $r^{\text{th}}$ mistake. Since $A$ is guaranteed to make at most $r$ mistakes on any sequence of inputs, it must be correct for the next input, regardless of the value of that input. Thus, in round $t+1$, $A$ can evaluate $f$ on any input using fewer than $\We(f)$ arithmetic operations. This is a contradiction of the definition of $\We(f)$, so we have $\Le_{F} \ge \We(F)$.

For the second case, suppose that $\We(F) = \infty$. For contradiction, suppose that $\Le_{F} < \We(F)$. As in the first case, we let $m = \optastd(F,\Le(F))$. Then, there exists some algorithm $A$ for learning $F$ with at most $m$ mistakes which uses at most $N < \infty$ arithmetic operations per round. Let $f \in F$ be a function for which $\We(f) > N$.

As in the first case, let $r$ be the maximum possible number of mistakes that $A$ makes on any sequence of inputs if the hidden function is $f$. There exists a sequence $S$ of inputs $x_1,x_2,\dots$ for which $A$ makes exactly $r$ mistakes on $S$ when the hidden function is $f$. Let $x_t$ be the input on which $A$ makes the $r^{\text{th}}$ mistake. Since $A$ is guaranteed to make at most $r$ mistakes on any sequence of inputs, it must be correct for the next input, regardless of the value of that input. Thus, in round $t+1$, $A$ can evaluate $f$ on any input using at most $N < \We(f)$ arithmetic operations. This is a contradiction, so we have $\Le_{F} \ge \We(F)$.
\end{proof}

Given the general lower bound in Theorem~\ref{triv_lower}, it is natural to investigate which families of functions $F$ have $\Le_{F} = \We(F)$. That is the focus of the next section. In the next result, we construct families of functions for which $\We(F) < \Le_{F}$.

\begin{thm}\label{thm_wf_linf}
There exists a class $F$ of functions with range $\{0,1\}$ for which $\We(F)$ is finite and $\Le_{F}=\infty$.
\end{thm}

\begin{proof}
For positive real numbers $x$ and $d$, define
\[f_x(d)=\left\lfloor dx\right\rfloor-2\left\lfloor\frac{dx}2\right\rfloor=\begin{cases}0&\lfloor d x \rfloor\text{ is even}\\1&\lfloor d x \rfloor\text{ is odd}\end{cases}\]
Consider the function class $F=\{f_x:x\in\mathbb R^+\}$. Since $f_x(d)$ can be computed using a few arithmetic operations for any $x$ and $d$, we have that $W(F)$ is finite. In either learning scenario, the adversary can provide inputs $2,4,8,16,\dots$. The correct outputs are the binary digits after the decimal point of $x$, which can vary freely. Thus, the adversary can claim that the learner gets infinitely many outputs wrong.
\end{proof}

We also exhibit families of functions $F$ which are easy to learn in the standard mistake-bound model without computational restrictions, but impossible to learn with caps on arithmetic operations.

\begin{theorem}\label{thm_ainf}
There exists a family $F$ of functions with $\We(F) = \infty$, $\optstd(F) = 0$, and $\Le_{F} = \infty$.
\end{theorem}

\begin{proof}
Let $u:\mathbb N\to\mathbb N$ be a computable function that takes an unbounded number of operations to compute as $n$ grows large. Define the function class $F=\{u\}$.
Since $F$ consists of just one computable function, its mistake bound in either scenario is $0$. However, since this function takes unboundedly many operations to compute as $n$ grows large, the learner can never achieve this mistake bound (or even any finite mistake bound, in fact) if given boundedly many operations per round. Thus, $\Le_{F}=\infty$.
\end{proof} 

\section{Families $F$ with $\Le_{F} = \We(F)$}\label{s:bounds}

In this section, we identify several families of functions that attain the lower bound in Theorem~\ref{triv_lower}. In particular, we prove a sufficient criterion for families of functions to attain the lower bound.

In the next result, we show that the lower bound in Theorem~\ref{triv_lower} is sharp when the domain of $F$ is finite. Note that even if its domain is finite, $F$ may still be infinite if it has infinite codomain.

\begin{thm}
    If $F$ has finite domain, then $\Le_{F} = \We(F) = 0$.
\end{thm}

\begin{proof}
    If $F$ has finite domain $X$, then every function $f \in F$ can be computed without using any arithmetic operations. In particular, since $X$ is finite, we can write $X = \left\{x_1, \dots, x_k \right\}$. For each $i$, let $y_i$ denote the value of $f(x_i)$. We can compute $f$ using finitely many if/then statements of the form ''If $x = x_i$, then return $y_i$.'' This corresponds to using branches for the outputs in the DAG. Since this does not use any arithmetic operations, we have $W(F) = 0$.

    To see that $\Le_{F} = 0$, let $y$ be a fixed element of the codomain of $F$. For each input $x$ from the adversary that has never been asked in an earlier round, the learner answers $y$. For each input $x$ from the adversary that has already been asked in an earlier round, the learner answers with the correct value of $f(x)$, which was revealed at the end of the first round in which the adversary gave the input $x$. Since $F$ has finite domain, the learner will only give at most $|X| < \infty$ wrong answers during the learning process. Since this learning algorithm does not use any arithmetic operations, we have $\Le_{F} = 0$.
\end{proof}

For the remainder of this section, we find more families where the lower bound in Theorem~\ref{triv_lower} is sharp. In particular, we define a property of families of functions $F$ which we call simple online-learnability, we show that being simple online-learnable implies that $\Le_{F} = W(F)$, and we find several families of functions which are simple online-learnable. 

We say that $F$ is \textit{simple online-learnable} if there exists an algorithm $A$ for learning $F$ in the standard learning scenario (strong reinforcement) such that
\begin{enumerate}
    \item whenever $A$ is given a new input $x$ in some round of the learning process, it chooses some $f$ in $F$ without using any arithmetic operations and then computes $f(x)$,
    \item $A$ uses no additional arithmetic operations in each round after making its guess if its guess is correct,
     \item $A$ uses at most $s$ additional arithmetic operations in each round after making its guess if its guess is wrong, and
    \item there exists some number $t < \infty$ for which $A$ makes at most $t$ total mistakes in the standard learning scenario without computational restrictions.
\end{enumerate}

Next, we prove a lemma that allows us to conclude that $\Le_{F} = W(F)$ whenever $F$ is simple online-learnable.

\begin{lem}\label{lem:sol}
If $F$ is simple online-learnable, then $\Le_{F} = W(F)$.
\end{lem}

\begin{proof}
The lower bound $\Le_{F} \ge W(F)$ follows from Theorem~\ref{triv_lower}. For the upper bound, observe that we can use the algorithm $A$ for learning $F$ from the definition of simple online-learnable to define a new algorithm $A’$ which alternates between two phases. Specifically, the algorithm $A’$ starts in phase 1, and when $A’$ is in phase 1, it simply uses $A$ to choose some $f \in F$ without using any arithmetic operations and then it computes $f(x)$, using at most $W(F)$ arithmetic operations. If the adversary ever says the answer was wrong when $A’$ was in phase 1, then $A’$ enters phase 2, in which it performs any additional calculations that $A$ would perform based on the wrong answer, answers an arbitrary $y \in Y$ without using any arithmetic operations for any inputs that are presented in subsequent rounds during phase 2, and forgets all results from the subsequent rounds occurring during phase 2. Once the additional calculations are complete, $A’$ goes back to phase 1. Note that if the additional calculations require at most $s$ arithmetic operations, then they take at most $\lceil s/W(F)\rceil$ rounds. Since $A$ makes at most $t$ errors in the standard learning scenario without computational restrictions for some $t < \infty$, $A’$ makes at most $t\left(1+ \lceil s/W(F) \rceil\right)$ errors when learning $F$ in the standard learning scenario with at most $W(F)$ arithmetic operations per round.
\end{proof}

Now, we apply Lemma~\ref{lem:sol} to some families of functions. As an illustrative example, consider the family $F_L(\mathbb{R}^n,\mathbb{R})$, which is the set of linear transformations from $\mathbb{R}^n$ to $\mathbb{R}$. In other words, for each $\mathbf v\in \mathbb{R}^n$ let the function $f_\mathbf v:\mathbb R^n\to\mathbb R$ be defined by $f_\mathbf v(\mathbf x)=\mathbf v\cdot\mathbf x$. Then $F_L(\mathbb{R}^n,\mathbb{R})=\{f_\mathbf v:\mathbf v\in\mathbb R^n\}$. For all positive integers $n$, we have $\optstd(F_L(\mathbb{R}^n,\mathbb{R}))=n$ \cite{almw, blum, shvaytser}.

Before stating the next result, we introduce some definitions and terminology. We say that function $f$ \textit{depends} on input variable $x_i$ if for some fixed values $c_1,\dots,c_{i-1},c_{i+1},\dots,c_k$ of the input variables to $f$ other than $x_i$, there exist $a,b$ with $a \neq b$ such that \[f(c_1,\dots,c_{i-1},a,c_{i+1},\dots,c_k) \neq f(c_1,\dots,c_{i-1},b,c_{i+1},\dots,c_k).\]

For each node $v$ in the directed graph $D_f$ corresponding to $f$, let $S_v$ denote the subset of input variables on which the output of $v$ depends. For example, $S_v$ consists of a single element if $v$ is an input node of $D_f$, and $S_v$ is some subset of the set of all input variables if $v$ is the output node of $D_f$.

If node $v$ only receives one input from some node $u$, then $S_v \subseteq S_u$. If node $v$ receives two inputs from some nodes $u$ and $w$, then $S_v \subseteq S_u \cup S_w$. We use the next lemma to obtain lower bounds on $\We(F)$ for several families of functions $F$. 

\begin{lem}\label{lem:depends}
For each natural number $k$, computing a function that depends on $k+1$ input variables requires at least $k$ binary arithmetic operations.   
\end{lem}

\begin{proof}
Suppose that the function $f$ has input variables $x_1, \dots, x_n$ and $f$ depends on $k+1$ of these input variables. In order to prove the bound, we construct a simple undirected graph $G$ based on the directed graph $D_f$. In particular, we treat each binary arithmetic operation in $D_f$ as an edge in $G$, which has nodes $x_1, \dots, x_n$. Specifically, we start with $G$ as the empty graph on $x_1, \dots, x_n$, and then we go through the binary arithmetic operations in $D_f$ in the order in which they are performed. For each binary operation at node $v \in V(D_f)$ which depends on the subset $S_v$, we add an edge between a single pair of input variables in $S_v$ which currently have no path between them in $G$. If no such pair exists, then we do not add an edge. At the end of the process, any two inputs $x_i, x_j$ which both appear in the same subset $S_v$ for some $v \in V(D_f)$ must be in the same connected component of $G$. Since the output node $q$ of $D_f$ satisfies $|S_q| = k+1$, we must have some connected component of $G$ with at least $k+1$ nodes. Thus, $G$ must have at least $k$ edges, so $D_f$ must use at least $k$ binary arithmetic operations.
\end{proof}

\begin{cor}
    For positive integers $n$, we have $\We(F_L(\mathbb{R}^n,\mathbb{R})) = \Theta(n)$.
\end{cor}

\begin{proof}
    The adversary can choose the hidden function $f_v$ where $v$ is the all-ones vector. Then, $\We(F_L(\mathbb{R}^n,\mathbb{R})) = \Omega(n)$ follows from Lemma~\ref{lem:depends}. On the other hand, we have $\We(F_L(\mathbb{R}^n,\mathbb{R})) = O(n)$ since any function in $F_L(\mathbb{R}^n,\mathbb{R})$ can be computed using at most $2n-1$ arithmetic operations.
\end{proof}

We apply Lemma~\ref{lem:sol} to the family $F = F_L(\mathbb{R}^n,\mathbb{R})$.

\begin{thm}\label{thm:linwf}
 For the family $F = F_L(\mathbb{R}^n,\mathbb{R})$, we have $\Le_{F} = \We(F)$.
\end{thm}

\begin{proof}
By Lemma~\ref{lem:sol}, it suffices to prove that $F$ is simple online-learnable, so we provide a learning algorithm which satisfies the properties of simple online-learnability. In each round, the learner keeps track of a linearly independent set $S$ of inputs. In the first round, when the adversary provides the input $x_1 \in \mathbb{R}^n$, the learner guesses $0$ for the output, regardless of the value of $x_1$, i.e., the learner's initial default assumption is that the hidden function is the zero function. In any round $j$ where the adversary says that the learner's guess was wrong, the learner receives the true output $f(x_j)$ from the adversary and the learner adds $x_j$ to $S$. After the learner adds $x_j$, suppose that the resulting set is $S = \{x_{i_1}, \dots, x_{i_k}\}$. The learner uses row reduction on the $k \times (n+1)$ matrix obtained from putting the concatenation of $x_{i_t}$ and $f(x_{i_t})$ in each row $t = 1, \dots, k$ to find a function $g \in F_L(\mathbb{R}^n,\mathbb{R})$ that is consistent with all previous inputs and outputs. 

Once the learner has found such a function $g$, they use $g$ for all future inputs until the adversary says they are wrong. If the current input $x’$ is in the span of the inputs used to find the function, then the learner’s guess will be the correct output and the learner makes no additional computations in the round after their guess. Otherwise, if the guess is wrong, then the current input must be linearly independent of the inputs in $S$, so the learner performs a new row reduction and obtains a new function $g’ \in F_L(\mathbb{R}^n,\mathbb{R})$ which is consistent with all previous inputs and outputs. Note that we must have $|S| \le n+1$ throughout the learning process, so there exists a positive integer $s = s(n)$ such that the learner uses at most $s$ additional arithmetic operations in each round after making its guess if the guess is wrong. They continue this strategy for the entire learning process. Since the learner can only get a guess wrong when the current input is not in the span of the current elements of $S$, they make at most $n$ total mistakes in the standard learning scenario without computational restrictions. Thus, $F$ is simple online-learnable, so we have $\Le_{F} = \We(F)$.
\end{proof}

\begin{cor}
 For the family $F = F_L(\mathbb{R}^n,\mathbb{R})$, we have $\Le_{F} = \Theta(n)$.
\end{cor}

Next, we obtain some upper and lower bounds on $\A_{F_{L}(\mathbb{R}^n,\mathbb{R})}$.

\begin{thm}
 For positive integers $n$ and $F = \text{F}_{\Le}(\mathbb{R}^n,\mathbb{R})$, we have $\A_{F} = O(n^3)$ and  $\A_{F} = \Omega(n)$.
\end{thm}

\begin{proof}
 First we prove that $\A_{F_L(\mathbb{R}^n,\mathbb{R})} = O(n^3)$. For each input $x_i$, the learner can use the strategy of checking whether $x_i$ is in the span of $x_1, \dots, x_{i-1}$. The learner maintains a subset $y_1, \dots, y_{k_{i-1}}$ of $x_1, \dots, x_{i-1}$ which spans $x_1, \dots, x_{i-1}$, so they only need to check whether $x_i$ is in the span of $y_1, \dots, y_{k_{i-1}}$. This can be done with row reduction using $O(n^3)$ arithmetic operations since $k_i \le n$ for all $i$. If $x_i$ is in the span of $y_1, \dots, y_{k_{i-1}}$, then row reduction yields a set of coefficients $c_1, \dots, c_{k_{i-1}}$ for which $x_i = \sum_{j = 1}^{k_{i-1}} c_j y_j$, so the learner obtains $f(x_i) = \sum_{j = 1}^{k_{i-1}} c_j f(y_j)$ using $O(n)$ operations. If $x_i$ is not in the span of $y_1, \dots, y_{k_{i-1}}$, then the learner guesses $f(x_i) = 0$ and sets $k_i = k_{i-1}+1$ and $y_{k_i} = x_i$. This completes the proof of the upper bound.

The lower bound is trivial, since we have $\A_{F_L(\mathbb{R}^n,\mathbb{R})} \ge \Le_{F_L(\mathbb{R}^n,\mathbb{R})} = \Omega(n)$.
\end{proof}

Next, we apply Lemma~\ref{lem:sol} to finite families of functions.

\begin{thm}
If $|F| < \infty$, then $\Le_{F} = \We(F)$.
\end{thm}

\begin{proof}
By Lemma~\ref{lem:sol}, it suffices to prove that $F$ is simple online-learnable. Since $|F| < \infty$, we can list the functions in $F$ as $f_1, \dots, f_z$ for some positive integer $z$. For the input $x_1$ in round $1$, the learner guesses $f_1(x_1)$. Suppose that the learner used the function $f_j$ to guess $f_j(x_i)$ in round $i$, and the adversary said that their guess was wrong. For the input $x_{i+1}$ in round $i+1$, the learner guesses $f_{j+1}(x_{i+1})$. Whenever the adversary says that the learner's guess was correct, the learner continues to use the same function for subsequent rounds until the adversary says that they made a wrong guess. This learning algorithm is guaranteed to make at most $|F|-1$ mistakes, since each function in $F$ can cause at most one mistake if it is not the hidden function, and it causes no mistakes if it is the hidden function. Thus, $F$ is simple online-learnable.
\end{proof}

Again using Lemma~\ref{lem:sol}, we show that it is possible to obtain similar results for some other families of functions.

\begin{thm}
    If $F$ is the set of all polynomials on $k$ variables where the $i^{\text{th}}$ variable has degree at most $d_i$, then \[\optstd(F) = \prod_{i = 1}^k (d_i+1).\]
\end{thm}

\begin{proof}
    Without loss of generality, let the polynomials in $F$ have variables $x_1, \dots, x_k$. To see that \[\optstd(F) \le \prod_{i = 1}^k (d_i+1),\] note that we can convert any input $(x_1, \dots, x_k)$ to a vector of length $\prod_{i = 1}^k (d_i+1)$ whose coordinates are the values of $\prod_{i = 1}^k x_i^{e_i}$ for all $0 \le e_i \le d_i$ with $1 \le i \le k$. Since learning a polynomial from $F$ amounts to learning its coefficients, we can use the same algorithm that we used to learn \[F_L(\mathbb{R}^{\prod_{i = 1}^k (d_i+1)},\mathbb{R})\] in Theorem~\ref{thm:linwf}. Since this algorithm makes at most $\prod_{i = 1}^k (d_i+1)$ mistakes, we have \[\optstd(F) \le \prod_{i = 1}^k (d_i+1).\] To see that \[\optstd(F) \ge \prod_{i = 1}^k (d_i+1),\] note that the adversary can give the learner the inputs $x_1 = t$ and $x_i = t^{\prod_{j = 1}^{i-1} (d_j+1)}$ for each $i > 1$ in the $t^{th}$ round of the learning process. After each guess from the learner, the adversary says they were wrong and that the correct answer was $0$ or $1$, whichever differs from the learner's guess. Since these inputs reduce the hidden function to a polynomial in $t$ of degree at most $D = -1+\prod_{i = 1}^k (d_i+1)$, and each input defines a different point on the polynomial, the adversary can force at least $D+1 = \prod_{i = 1}^k (d_i+1)$ mistakes from the learner before the polynomial is determined. 
\end{proof}

\begin{thm}
    If $F$ is the set of all polynomials on $k$ variables where the $i^{\text{th}}$ variable has degree at most $d_i$, then $\Le_{F} = \We(F)$.
\end{thm}

\begin{proof}
    We can view a polynomial on $k$ variables where the $i^{\text{th}}$ variable has degree at most $d_i$ as a vector of length $\prod_{i = 1}^k (d_i+1)$ where the coordinates of the vector are the corresponding coefficients of the polynomial. Note that we can use the same learning algorithm as in Theorem~\ref{thm:linwf} to learn $F$ by converting each input of values for the $k$ variables $x_1, \dots, x_k$ into a vector of values for $\prod_{i = 1}^k x_i^{e_i}$ for all $0 \le e_i \le d_i$ with $1 \le i \le k$. Thus, $F$ is simple online-learnable, so $\Le_{F,std} = \We(F)$ by Lemma~\ref{lem:sol}.
\end{proof}

\section{Learning neural networks in the mistake-bound model}\label{sec:1layer}

In this section, we discuss the online learning of neural networks in the mistake-bound model. We start with a number of results about $1$-layer neural networks, and then we consider neural networks with multiple layers.

We show that several families of $1$-layer neural networks with a variety of different activation functions (relu, leaky relu, sigmoid, and tanh) can be learned in the mistake-bound model with operation caps. However, with two or more layers, the learner cannot guarantee finite error in the mistake-bound model, even with only a few neurons.

\subsection{$1$-layer neural networks}

In this subsection, we focus on the online learning of $1$-layer neural networks. In particular, for a given activation function $a: \mathbb{R} \rightarrow \mathbb{R}$, let $F_{n,a}$ be the family of functions $g: \mathbb{R}^n \rightarrow \mathbb{R}$ of the form $g(x) = a(f(x)+b)$ for any $b \in \mathbb{R}$ and $f \in F_L(\mathbb{R}^n,\mathbb{R})$. 

\begin{lem}
    Suppose that the activation function $a: \mathbb{R} \rightarrow \mathbb{R}$ is invertible, and there exists a positive integer $r$ such that $a$ can be computed for any input using at most $r$ arithmetic operations. Then, for positive integers $n$, we have $\We(F_{n,a}) = \Theta(n)$.
\end{lem}

\begin{proof}
    The adversary can choose the hidden function $a(f_v(x))$ where $v$ is the all-ones vector. Then, $\We(F_{n,a}) = \Omega(n)$ follows from Lemma~\ref{lem:depends} since $a$ is invertible. On the other hand, we have $\We(F_{n,a}) = O(n)$ since any function in $F_{n,a}$ can be computed using at most $2n+r$ arithmetic operations.
\end{proof}

In the next theorem, we prove a general result that can be applied to a number of different activation functions. 

\begin{thm}\label{thm:acta}
Suppose that the activation function $a: \mathbb{R} \rightarrow \mathbb{R}$ is invertible, and there exists a positive integer $r$ such that both $a$ and its inverse $a^{-1}$ can be computed for any input using at most $r$ arithmetic operations. Then, we have $\Le_{F_{n,a}} = \We(F_{n,a}) = \Theta(n)$.
\end{thm}

\begin{proof}
By Lemma~\ref{lem:sol}, it suffices to prove that $F_{n,a}$ is simple online-learnable, so we provide a learning algorithm which satisfies the properties of simple online-learnability. In each round, the learner keeps track of a linearly independent set $S$ of vectors. In the first round, when the adversary provides the input $x_1 \in \mathbb{R}^n$, the learner guesses $a(0)$ for the output, regardless of the value of $x_1$, i.e., the learner's initial default assumption is that the part of the hidden function inside the activation function is the zero function. In any round $j$ where the adversary says that the learner's guess was wrong, the learner receives the true output $y_j$ from the adversary. If $x'_j$ denotes the vector obtained from the input $x_j$ by appending $1$, the learner adds $x'_j$ to $S$. After the learner adds $x'_j$, suppose that the resulting set is $S = \{x'_{i_1}, \dots, x'_{i_k}\}$. The learner uses row reduction on the $k \times (n+2)$ matrix obtained from putting the concatenation of $x'_{i_t}$ and $a^{-1}(y_{i_t})$ in each row $t = 1, \dots, k$ to find the coefficients for a linear function from $\mathbb{R}^n$ to $\mathbb{R}$ of the form $f(x)+b$ for some $f \in F_L(\mathbb{R}^n,\mathbb{R})$ and $b \in \mathbb{R}$ such that $g(x) = a(f(x)+b)$ is consistent with all previous inputs and outputs. 

Once the learner has found a function $g \in F_{n,a}$ such that $g(x) = a(f(x)+b)$ is consistent with all previous inputs and outputs, they use $g$ for all future inputs until the adversary says they are wrong. If $x'_j$ is in the span of the current elements of $S$, then the learner’s guess will be the correct output and the learner makes no additional computations in the round after their guess. Otherwise, the learner constructs a new matrix and performs a new row reduction, obtaining a new function $g’ \in F_{n,a}$ which is consistent with all previous inputs and outputs. They continue this strategy for the entire learning process. Since the learner can only get a guess wrong when $x'_j$ is not in the span of the current elements of $S$, the learner makes at most $n+1$ total mistakes in the standard learning scenario without computational restrictions. Furthermore, note that we must have $|S| \le n+1$ throughout the learning process, and both $a$ and $a^{-1}$ can be computed for any input using at most $r$ arithmetic operations, so there exists a positive integer $s = s(n,r)$ such that the learner uses at most $s$ additional arithmetic operations in each round after making its guess if the guess is wrong. Thus, $F_{n,a}$ is simple online-learnable, so we have $\Le_{F_{n,a}} = \We(F_{n,a})$. 
\end{proof}

We apply the last theorem to $1$-layer neural networks using a leaky relu activation function. Specifically, for $\alpha > 0$, let \[\relu_{\alpha}(x) =  \begin{cases}        x & x > 0 \\       \alpha x & x \le 0     \end{cases} .\] 

\begin{cor}
For all real numbers $\alpha > 0$, we have $\Le_{F_{n,\relu_{\alpha}}} = \We(F_{n,\relu_{\alpha}}) = \Theta(n)$.
\end{cor}

\begin{thm}\label{opt_sig}
    Suppose that the activation function $a: \mathbb{R} \rightarrow \mathbb{R}$ is invertible. Then, $\optstd(F_{n, a}) = n+1$.
\end{thm}

\begin{proof}
     Suppose that the hidden function is $h \in F_{n, a}$. For each input $x_j \in \mathbb{R}^n$, define $x'_j \in \mathbb{R}^{n+1}$ to be the vector obtained from $x_j$ by appending a $1$. Given $x_j$, the learner checks whether $x'_j$ is in the span of the previous $x'_1, \dots, x'_{j-1}$. If not, then the learner guesses that the output is $a(0)$. Otherwise, if \[x'_j = \sum_{i = 1}^{j-1} c_i x'_i,\] then the learner guesses that the output is \[a\left(\sum_{i = 1}^{j-1} c_i a^{-1}(h(x'_i))\right).\] The learner always gets the correct answer when $x'_j$ is in the span of the previous $x'_1, \dots, x'_{j-1}$, so they only get the wrong answer at most $n+1$ times. Thus, $\optstd(F_{n, a}) \le n+1$. 
    
    To see that $\optstd(F_{n, a}) \ge n+1$, note that the adversary can give the all-$0$ vector as the input in round $1$ and $e_i$ as the input in round $i+1$ for $1 \le i \le n$, where $e_i$ is the standard unit vector with $i^{\text{th}}$ coordinate $1$ in $\mathbb{R}^n$. Regardless of the learner's answer, the adversary always says it is wrong. If the learner's answer was $a(0)$, then the adversary says that the answer was $a(1)$. Otherwise, the adversary says that the answer was $a(0)$. The $n+1$ vectors obtained from the all-$0$ vector and the standard unit vector by appending a $1$ are linearly independent, so there exists a function in $F_{n, a}$ which is consistent with the adversary's answers.
\end{proof}

\begin{cor}
    For all real numbers $\alpha > 0$, we have $\optstd(F_{n, \relu_{\alpha}}) = n+1$.
\end{cor}

The sigmoid activation function is given by $\sigma: \mathbb{R} \rightarrow \mathbb{R}$ with \[\sigma(x) = \frac{1}{1 + e^{-x}}.\] This function is invertible, since \[x = \ln{\frac{\sigma(x)}{1-\sigma(x)}}.\] Thus, we obtain the following corollaries of Theorem~\ref{thm:acta} and Theorem\ref{opt_sig}.

\begin{cor}
    If $\sigma: \mathbb{R} \rightarrow \mathbb{R}$ is the sigmoid activation function, then $\Le_{F_{n,\sigma}} = \We(F_{n,\sigma}) = \Theta(n)$.
\end{cor} 

\begin{cor}
    If $\sigma: \mathbb{R} \rightarrow \mathbb{R}$ is the sigmoid activation function, then $\optstd(F_{n, \sigma}) = n+1$.
\end{cor} 

The hyperbolic tangent activation function is given by $\tanh: \mathbb{R} \rightarrow \mathbb{R}$ with \[\tanh(x) = \frac{e^x - e^{-x}}{e^x + e^{-x}}.\] This function is invertible, since \[x = \frac{1}{2}\ln{\frac{1+\tanh(x)}{1-\tanh(x)}}.\] Thus, we obtain the following corollaries of Theorem~\ref{thm:acta} and Theorem\ref{opt_sig}.

\begin{cor}
    If $\tanh: \mathbb{R} \rightarrow \mathbb{R}$ is the hyperbolic tangent activation function, then $\Le_{F_{n,\tanh}} = \We(F_{n,\tanh}) = \Theta(n)$.
\end{cor} 

\begin{cor}
    If $\tanh: \mathbb{R} \rightarrow \mathbb{R}$ is the hyperbolic tangent activation function, then $\optstd(F_{n, \tanh}) = n+1$.
\end{cor} 

The standard relu activation function is given by $\relu: \mathbb{R} \rightarrow \mathbb{R}$ with \[\relu(x) =  \begin{cases}        x & x > 0 \\       0 & x \le 0     \end{cases}.\] Theorem\ref{thm:linwf} does not apply to the standard relu activation function,  since $\relu$ is not invertible. However, in the next result, we show how to learn $1$-layer neural networks that use a standard relu activation function. 

\begin{thm}\label{thm:relu_1layer}
    For the standard relu activation function, we have $\optstd(F_{n, \relu}) = n+1$.
\end{thm}

\begin{proof}
    Suppose that the hidden function is $h \in F_{n, \relu}$. For each input $x_j \in \mathbb{R}^n$, define $x'_j \in \mathbb{R}^{n+1}$ to be the vector obtained from $x_j$ by appending a $1$. Given the input $x_j$, the learner checks whether $x'_j$ is in the span of the previous $x'_{t_1}, \dots, x'_{t_k}$ which had positive outputs. If not, then the learner guesses the output is $0$. Otherwise, if \[x'_j = \sum_{i = 1}^{k} c_i x'_{t_i},\] then the learner guesses that the output is \[\relu\left(\sum_{i = 1}^{j-1} c_i h(x'_{t_i})\right).\] The learner always gets the correct answer when $x'_j$ is in the span of the previous $x'_{t_1}, \dots, x'_{t_k}$ with positive outputs, and the learner always gets the correct answer when the output is $0$, so they only get the wrong answer at most $n+1$ times. Thus, $\optstd(F_{n, \relu}) \le n+1$. 
    
    To see that $\optstd(F_{n, \relu}) \ge n+1$, note that the adversary can use the same strategy as in Theorem~\ref{opt_sig}.
\end{proof}

Note that the method in the last proof does not imply that $\Le_{F_{n, \relu}} < \infty$, since the number of arithmetic operations used by the learner increases with the number of rounds. In the next result, we show how to modify the learner strategy to obtain $\Le_{F_{n, \relu}} < \infty$.

\begin{thm}
For positive integers $n$, we have $\Le_{F_{n, \relu}} \le \A_{F_{n, \relu}} = O(n^3)$.
\end{thm}

\begin{proof}
We use a similar strategy to the last proof. Suppose that the hidden function is $h \in F_{n, \relu}$.

For each input $x_j \in \mathbb{R}^n$, define $x'_j \in \mathbb{R}^{n+1}$ to be the vector obtained from $x_j$ by appending a $1$. Given the input $x_j$, let $S_j$ be the set of the previous $x'_{t_1}, \dots, x'_{t_k}$ which had positive outputs. The learner checks whether $x'_j$ is in the span of $S_j$ by maintaining a subset $T_j \subseteq S_j$ which spans $S_j$, so they only need to check whether $x'_j$ is in the span of $T_j$. This can be done with row reduction using $O(n^3)$ arithmetic operations since $|T_j| \le n+1$ for all $j$. 

If $x'_j$ is not in the span of $T_j$, then the learner guesses $0$. Otherwise, if \[x'_j = \sum_{r \in T_j} c_r r,\] then the learner guesses that the output is \[\relu\left(\sum_{r \in T_j} c_r h(r)\right).\] The learner always gets the correct answer when $x'_j$ is in the span of $S_j$, and the learner always gets the correct answer when the output is $0$, so they only get the wrong answer at most $n+1$ times. Checking whether $x'_j$ is in the span of $T_j$ takes $O(n^3)$ arithmetic operations, and calculating the learner's guess takes an additional $O(n)$ arithmetic operations, so the learner uses at most $O(n^3)$ arithmetic operations per round.
\end{proof}

We also consider $1$-layer neural networks in which the output is a probability distribution. For example, consider a neural network that classifies handwritten digits into ten classes. The output will be a vector of length $10$, such that the entries are positive and add up to $1$. In the next result, we determine the exact cost of learning these families of functions in the standard learning scenario. 




\begin{thm}
    If $\sigma_k: \mathbb{R} \rightarrow \mathbb{R}$ is the softmax activation function with $k > 1$ classes, then $\optstd(F_{n, \sigma_k}) = n+1$.
\end{thm} 

\begin{proof}
    Let $x = (x_1, \dots, x_n) \in \mathbb{R}^n$ be the initial input, and let $x'\in \mathbb{R}^{n+1}$ be the vector obtained from $x$ by appending a $1$. If $x'$ is linearly independent of $z'$ for all past inputs $z$, then the learner answers with $\frac{1}{n}\mathbf{1}_n$. Otherwise, suppose that the inputs to the softmax are $f_1(x)+b_1, \dots, f_k(x)+b_k$, where $f_i \in F_L(\mathbb{R}^n,\mathbb{R})$ and $b_i \in \mathbb{R}$ for each $i$. Each $f_i(x)$ can be written as $\sum_{j = 1}^n a_{i j} x_j$ for some real numbers $a_{i j}$. Note that we have $\sigma_k(f_1(x), \dots, f_k(x)) = (y_1, \dots, y_k)$ if and only if \[(b_i - b_1)+\sum_{j = 1}^n (a_{i j} - a_{1 j}) x_j = \ln(y_i)-\ln(y_1)\] for all $i > 1$.
    
    Thus, instead of the learner trying to solve for all of the $a_{i j}$ and $b_i$, they only need to solve for the differences $c_{i j} = a_{i j}-a_{1 j}$ and $d_i = b_{i} - b_{1}$ for each $i > 1$. 

    In each round, the learner uses row reduction on $k-1$ matrices, each of which is constructed using all past inputs and outputs. More precisely, if the adversary gives the correct answer $y_1, \dots, y_k$ for $\sigma_k(f_1(x), \dots, f_k(x))$ in round $t$, then the $t^{\text{th}}$ row of the $i^{\text{th}}$ matrix consists of $x$ concatenated with $1, \ln{y_i}-\ln{y_1}$. 

    The learner performs row reduction to find values of $c_{i j}$ and $d_i$ that are consistent with the past inputs and outputs. They use these values on the current input to calculate the final $k$ outputs $y_1, \dots, y_k$ of the softmax, by solving for $y_1, \dots, y_k$ such that \[y_i = y_1 e^{d_i+\sum_{j = 1}^n c_{i j} x_j}\] for all $i > 1$ and \[\sum_{i = 1}^k y_i = 1.\] If $x'$ can be written as a linear combination of $z'$ for the past inputs $z$, then the learner's answer will be correct. At most $n+1$ of the inputs $x$ can have $x'$ linearly independent of $z'$ for the past inputs $z$. Thus, the learner makes at most $n+1$ mistakes. 

    To see that the adversary can force $n+1$ mistakes, we can use a similar strategy to the other proofs in this section. The adversary asks the $n+1$ inputs consisting of the all-$0$ vector and each of the standard unit vectors. Regardless of the learner's answer, the adversary always says that it is wrong. If the learner's answer was $\sigma_k(0,\dots,0)$, then the adversary says that the answer was $\sigma_k(0,\dots,0,1)$. Otherwise, the adversary says that the answer was $\sigma_k(0,\dots,0)$.
\end{proof}

\subsection{Neural networks with multiple layers}

In Theorem~\ref{thm:relu_1layer}, we showed that it is possible in the mistake-bound model to learn $1$-layer neural networks that use a standard relu activation function. However, we show in the next theorem that it is not possible in the mistake-bound model to learn $2$-layer neural networks that use a standard relu activation function.

\begin{thm}
    If $F$ is the family of $2$-layer neural networks with a single input, $4$ neurons in the hidden layer, a single neuron in the output layer, and standard relu activation functions, then we have $\optstd(F) = \infty$.
\end{thm}

\begin{proof}
    Let $I_{[0,c]}: \mathbb{R} \rightarrow \mathbb{R}$ denote the indicator function such that $I_{[0,c]}(x) = 1$ if $x \in [0, c]$ and $I_{[0,c]}(x) = 0$ otherwise. Note that $I_{[0,c]}$ can be approximated by a linear combination of four relu activation functions: \[\frac{1}{\varepsilon}(\relu(x-(c+\varepsilon))+\relu(x+\varepsilon)-\relu(x-c)-\relu(x)),\] where the approximation error goes to $0$ as $\varepsilon \rightarrow 0$. In particular, the approximation function is equal to $1$ for all $x \in [0, c]$ and $0$ for all $x \le \varepsilon$ and $x \ge c+\varepsilon$. Any such approximation function can be computed by a $2$-layer neural network with a single input $x$, $4$ neurons in the hidden layer, a single neuron in the output layer, and standard relu activation functions.

    Suppose that $f$ is the hidden function. Assume that the adversary restricts $f$ so that $f(0) = 1$ and $f(1) = 0$. The adversary uses the following strategy. In the first round, they set $x_1 = \frac{1}{2}$ and ask the learner for $f(x_1)$. The adversary says that the correct answer is $0$ or $1$, whichever is farther from the learner's answer. If $f(x_1) = 1$, then the adversary sets $x_2 = \frac{3}{4}$ and repeats the process indefinitely. Otherwise, if $f(x_1) = 0$, then the adversary sets $x_2 = \frac{1}{4}$ and repeats the process indefinitely. The adversary forces an error of at least $\frac{1}{2}$ in each round, so we have $\optstd(F) = \infty$.
\end{proof}

As a result, we have found another family of functions $F$ for which $\We(F) \neq \Le_{F}$.

\begin{cor}\label{cor_wf_linf_2layer_relu}
If $F$ is the family of $2$-layer neural networks with a single input, $4$ neurons in the hidden layer, a single neuron in the output layer, and standard relu activation functions, then $\We(F)$ is finite and $\Le_{F}=\infty$.
\end{cor}

In fact, we can extend the last result to neural networks with leaky relu activation functions. In order to approximate $I_{[0,c]}$, we use the following sum of leaky relu activation functions: \[\frac{1}{(1-\alpha)\varepsilon}(\relu_{\alpha}(x-(c+\varepsilon))+\relu_{\alpha}(x+\varepsilon)-\relu_{\alpha}(x-c)-\relu_{\alpha}(x)).\] When $\alpha = 0$, note that this is the same approximation function that we used in the last proof. As with the approximation function in the last proof, this approximation function is equal to $1$ for all $x \in [0, c]$ and $0$ for all $x \le \varepsilon$ and $x \ge c+\varepsilon$. Thus, the adversary can use the same strategy as in the last proof. Since any such approximation function can be computed by a $2$-layer neural network with a single input $x$, $4$ neurons in the hidden layer, a single neuron in the output layer, and leaky relu activation functions of the form $\relu_{\alpha}$, we have the following result.

\begin{thm}
    If $F_{\alpha}$ is the family of $2$-layer neural networks with a single input, $4$ neurons in the hidden layer, a single neuron in the output layer, and leaky relu activation functions $\relu_{\alpha}$ with $0 < \alpha < 1$, then we have $\optstd(F_{\alpha}) = \infty$.
\end{thm}

Thus, we have another family of functions $F$ for which $\We(F) \neq \Le_{F}$.

\begin{cor}\label{cor_wf_linf_2layer_leaky}
If $F_{\alpha}$ is the family of $2$-layer neural networks with a single input, $4$ neurons in the hidden layer, a single neuron in the output layer, and leaky relu activation functions $\relu_{\alpha}$ with $0 < \alpha < 1$, then $\We(F_{\alpha})$ is finite and $\Le_{F_{\alpha}}=\infty$.
\end{cor}


\section{Comparing the worst-case error in different scenarios}\label{s:variants}

In this section, we prove a number of inequalities in which we bound the worst-case error for various learning scenarios in terms of the worst-case error for the standard learning scenario. We start with some inequalities for agnostic learning where we do not have any restrictions on the number of arithmetic operations. One of these inequalities for agnostic learning with weak reinforcement resolves a problem from \cite{filmus} and \cite{gt24}. Then, we apply the same idea to prove analogous bounds for agnostic learning with operations caps, as well as several other learning scenarios with operation caps. 

\subsection{Agnostic learning without operation caps}

In this subsection, we prove new upper bounds for the worst-case error in mistake-bounded agnostic learning. Our first two results focus on agnostic learning with strong reinforcement, while our third result focuses on agnostic learning with weak reinforcement. For agnostic learning with strong reinforcement, we show that \[\optagst(F,\eta) = O(\optstd(F) + \eta\ln{k})\] and \[\optagst(F,\eta)\leq (\optstd(F)+1)(\eta+1)-1.\] It is clear that there exist $F$ with \[\optagst(F,\eta) = \Omega(\optstd(F)+\eta),\] e.g., any $F$ with $|F| > 1$ since it suffices to have \[\optagst(F,\eta) \ge \optstd(F) \text{ and } \optagst(F,\eta)\ge \eta,\] but it is still an open problem to determine the maximum possible value of $\optagst(F,\eta)$ with respect to $\optstd(F)$, $\eta$, and the size $k$ of the codomain. For agnostic learning with weak reinforcement, we show that \[\optagwe(F,\eta) = O((k \ln{k})\optstd(F) + k \eta),\] which is sharp up to a constant factor and resolves a problem from \cite{filmus,gt24}. 

In the following proofs, we use several variants of the weighted majority voting algorithm for the learner's strategy, similar to some earlier papers on the mistake-bound model \cite{auer, gt24, long}. In each proof, we assume that there exists an algorithm $A_s$ for learning the given family of functions $F$ optimally in the standard online learning scenario. The learner's strategy starts with a single copy of $A_s$ that has weight $1$. As the learning process progresses, each copy of $A_s$ splits into new copies of $A_s$ whenever there is a mistake. The new copies may receive different information from the learner and they may have different weights. In each round, when the adversary provides the learner with an input, the learner guesses that the correct output is the one with the greatest weighted vote from the current copies of $A_s$. Let $A_b$ be the function which returns the guess with the greatest weighted vote in each round. By tailoring the weights of the copies for each scenario, we are able to derive upper bounds on the worst-case error of $A_b$.

\begin{thm}\label{ag_strong}
For all families of functions $F$ with codomain $\left\{0,1, \dots, k-1\right\}$, we have \[\optagst(F,\eta) = O(\optstd(F) + \eta\ln{k}).\]    
\end{thm}

\begin{proof}
Whenever the adversary says that $A_b$ got the wrong answer for the output of $f(x)$ and the adversary says $f(x) = q$, each copy of $A_s$ that voted for $f(x) \neq q$ splits into at most $k$ new copies, one that gets $f(x) = q$ and the others that get $f(x) = i$ for each $i \neq q$. We only include copies that have received outputs in all rounds up to the current round that are consistent with some function in $F$. If the old copy had weight $w$, then the new copy with $f(x) = q$ has weight $\frac{w}{3}$ and each of the other new copies has weight $\frac{w}{3k}$. The old copies that voted for $f(x) = q$ are unchanged.

First, observe that $\frac{1}{3}+(k-1)\frac{1}{3k} < \frac{2}{3}$ and the copies that voted for $f(x) = q$ account for at most half of the total weight. Thus, if the adversary says that $A_b$ makes a mistake in round $t$, and the total weight of all active copies is $W_t$ right before round $t$, then we have \[W_{t+1} \le \frac{W_t}{2}+\left(\frac{2}{3}\right)\frac{W_t}{2} =  \left(\frac{5}{6}\right)W_t.\] Moreover, there must always be some active copy of $A_s$ which receives the correct information, so this copy must have weight at least $\left(\frac{1}{3}\right)^{\optstd(F)}\left(\frac{1}{3k}\right)^{\eta}$. Thus, for all $t$ we have \[W_t \ge \left(\frac{1}{3}\right)^{\optstd(F)}\left(\frac{1}{3k}\right)^{\eta}.\] If $m$ denotes the number of mistakes that $A_b$ makes, then we have \[\left(\frac{5}{6}\right)^m \ge \left(\frac{1}{3}\right)^{\optstd(F)}\left(\frac{1}{3k}\right)^{\eta},\] which implies that \[\optagst(F,\eta) \le \frac{\optstd(F) \ln{\left(\frac{1}{3}\right)}+\eta \ln{\left(\frac{1}{3k}\right)}}{\ln{\left(\frac{5}{6}\right)}} = O(\optstd(F) + \eta\ln{k}).\]
\end{proof}

Next, we prove an alternative upper bound on $\optagst(F,\eta)$. The previous upper bound is sharper for $\optstd(F) = \omega(\ln{k})$, while the next upper bound is sharper for $\optstd(F) = o(\ln{k})$. 

\begin{thm}\label{ag_strong_alt}
For all families of functions $F$, we have \[\optagst(F,\eta)\leq (\optstd(F)+1)(\eta+1)-1.\]
\end{thm}

\begin{proof}
The learner runs an algorithm for the standard learning scenario, $A$, that achieves a mistake bound of $M$. 
After the adversary claims the the learner has made the $(M+1)^\text{st}$ mistake, it is guaranteed that the adversary has lied. The learner then ``forgets" all previous rounds, and runs $A$ again, until the adversary claims the learner has made $M+1$ more mistakes, and so on. If the learner makes at least $(M+1)(\eta+1)$ mistakes, then the adversary has lied at least $\eta+1$ times, contradiction. Thus, the learner makes at most $(M+1)(\eta+1)-1$ mistakes, as desired. 
\end{proof}

The next result is sharp up to a constant factor, matching a lower bound from \cite{gt24}. 

\begin{thm}\label{ag_weak}
For all families of functions $F$ with codomain $\left\{0,1, \dots, k-1\right\}$, we have \[\optagwe(F,\eta) = O((k \ln{k})\optstd(F) + k \eta).\]    
\end{thm}

\begin{proof}
Whenever the adversary says that $A_b$ got the wrong answer for the output of $f(x)$ and $A_b$ guessed $f(x) = q$, each copy of $A_s$ that voted for the wrong answer splits into at most $k$ new copies, one that gets $f(x) = q$ and the others that get $f(x) = i$ for each $i \neq q$. If the old copy had weight $w$, then the new copy with $f(x) = q$ has weight $\frac{w}{3}$ and each of the other new copies has weight $\frac{w}{3k}$. New copies are only included if they are consistent with the results of the learning process so far, otherwise the old copy splits into fewer than $k$ copies. The old copies of $A_s$ that voted for $f(x) \neq q$ are unchanged.

First, observe that $\frac{1}{3}+(k-1)\frac{1}{3k} < \frac{2}{3}$ and the copies that voted for $f(x) = q$ account for at at least a $\frac{1}{k}$ fraction of the total weight. Thus, if $A_b$ makes a mistake in round $t$, and the total weight of all active copies is $W_t$ right before round $t$, then we have \[W_{t+1} \le \frac{(k-1)W_t}{k}+\left(\frac{2}{3}\right)\frac{W_t}{k} =  \left(1-\frac{1}{3k}\right)W_t.\] Moreover, there must always be some active copy of $A_s$ which receives the correct information, so this copy must have weight at least $\left(\frac{1}{3k}\right)^{\optstd(F)}\left(\frac{1}{3}\right)^{\eta}$. Thus, for all $t$ we have \[W_t \ge \left(\frac{1}{3k}\right)^{\optstd(F)}\left(\frac{1}{3}\right)^{\eta}.\] If $m$ denotes the number of mistakes that $A_b$ makes, then we have \[\left(1-\frac{1}{3k}\right)^m \ge \left(\frac{1}{3k}\right)^{\optstd(F)}\left(\frac{1}{3}\right)^{\eta},\] which implies that \[\optagst(F,\eta) \le \frac{\optstd(F) \ln{\left(\frac{1}{3k}\right)}+\eta \ln{\left(\frac{1}{3}\right)}}{\ln{\left(1-\frac{1}{3k}\right)}} = O((k \ln{k})\optstd(F) + k \eta).\]
\end{proof}

Note that while the last result is sharp up to a constant multiplicative factor, we have not attempted to optimize the constant in the new upper bound.

\subsection{Scenarios with operation caps}

In the next proof, we extend the bound $\optbs(F) \le (1+o(1)) (k \ln{k}) \optstd(F)$ of Long \cite{long} to learning scenarios with restrictions on the number of arithmetic operations.

\begin{thm}
For any family $F$ of functions with codomain of size $k$ and any natural number $a$, let $M = \optastd(F,a)$. Then we have \[\optabs(F,(k-1)^M (a+2)) \le (1+o(1)) (k \ln{k}) M,\] where the $o(1)$ is with respect to $k$.
\end{thm}

\begin{proof} In this learning algorithm, we start with an algorithm $A_s$ which learns $F$ in the standard learning scenario with at most $\optastd(F,a)$ mistakes, using only at most $a$ arithmetic operations per round. Using copies of $A_s$, we construct a weighted majority voting algorithm $A_b$ for the restricted bandit scenario where the learner uses at most $(k-1)^M(a+2)$ arithmetic operations per round. 

For any input $x$ from the adversary, we give it to the currently active copies of $A_s$, and they each produce guesses for the output. If there are a total of $Q$ active copies, then this part of the learning algorithm requires at most $Q a$ arithmetic operations. Our learning algorithm $A_b$ determines the guess with the greatest total weight and returns it, which requires at most $Q$ arithmetic operations. If the adversary says that $A_b$ is incorrect, we deactivate the copies of $A_s$ that voted for the winning output, we make multiple new clones of each of them, and we obtain their new weights by multiplying the weight of the original copy by a scalar $\alpha$ which only depends on $k$ and can be computed before the start of the learning process. If there are a total of $Q$ active copies, then this part of the learning algorithm requires at most $Q$ arithmetic operations.

Since a copy of $A_s$ that makes a mistake will deactivate and produce at most $k-1$ new active copies, there are at most $(k-1)^T$ active copies at any time, where $T$ is the number of mistakes made by $A_b$. Thus, $Q \le (k-1)^T$, so the number of arithmetic operations in each round of the learning algorithm $A_b$ is at most $Qa+Q+Q \le (k-1)^M (a+2)$. Thus, we have \[\optabs(F,(k-1)^M (a+2)) \le (1+o(1)) (k \ln{k}) M,\] where the $o(1)$ is with respect to $k$.
\end{proof}

In the next result, we obtain a bound for the worst-case error in the delayed ambiguous reinforcement learning scenario in terms of the worst-case error in the standard learning scenario. To do this, we bound the number of arithmetic operations used in the weighted majority algorithm of Geneson and Tang \cite{gt24} to prove that \[\optbs(\cartr(F)) \le \optambr(F) \le (1+o(1))(k^r \ln{k})\optstd(F).\]

\begin{thm}
For any family $F$ of functions with codomain of size $k$ and any natural number $a$, let $M = \optastd(F,a)$. If $T = (r(k-1))^M (r a + r + 1)$, then we have \[\optabs(\cartr(F),T) \le \optaambr(F,T) \le (1+o(1))(k^r \ln{k})M,\] where the $o(1)$ is with respect to $k$.
\end{thm}

\begin{proof}
    The upper bound of $(r(k-1))^M (r a + r + 1)$ follows from an analogous argument to the previous theorem, with the difference being that active copies of $A_s$ which make a mistake may produce up to $r(k-1)$ new copies and each copy of $A_s$ may be run up to $r$ times per round. Using the same notation as in the last proof, we have $Q \le (r(k-1))^M$ and the number of arithmetic operations in each round of the learning algorithm $A_b$ is at most $Q r a + Q r +Q \le (r(k-1))^M (r a + r + 1)$.
\end{proof}

Now, we turn to agnostic learning. With both strong and weak reinforcement, we use the same voting strategies as in the previous subsection.

\begin{thm}\label{thm_ag_st_cap}
For all families of functions $F$ with codomain $\{0,1,\dots,k-1\}$ such that $\text{opt}_{\text{cap,std}}(F,a)=M$, we have \[\text{opt}_{\text{cap,ag,strong}}(F,\eta, ak^{O(M+\eta\ln k)})=O(M+\eta\ln k).\]
\end{thm}

\begin{proof}
Use the same voting strategy as in Theorem~\ref{ag_strong}, where each copy $A$ is an algorithm that achieves a mistake bound of $M$ using at most $a$ operations per round, in the standard learning scenario. After each of at most $O(M+\eta\ln k)$ mistakes, some of the copies split into at most $k$ copies, so there are at most $k^{O(M+\eta\ln k)}$ copies at any given time. Each of these copies uses at most $a$ operations to run each round, the weighted vote takes only one operation per copy, and the weight updates take at most one operation per copy, so the voting strategy uses $ak^{O(M+\eta\ln k)}$ operations per round.
\end{proof}

By an analogous proof to Theorem~\ref{thm_ag_st_cap}, we obtain the following result using the same voting strategy as in Theorem~\ref{ag_weak}.

\begin{thm}
For all families of functions $F$ with codomain $\{0,1,\dots,k-1\}$ such that $\text{opt}_{\text{cap,std}}(F,a)=M$, we have \[\text{opt}_{\text{cap,ag,weak}}(F,\eta,ak^{O((k \ln{k})M + k \eta)})=O((k \ln{k})M + k \eta).\]
\end{thm}

The next result follows from using the same strategy as in Theorem~\ref{ag_strong_alt}. 

\begin{thm}
For all families of functions $F$ such that $\text{opt}_{\text{cap,std}}(F,a)=M$, we have \[\text{opt}_{\text{cap,ag,strong}}(F,\eta,a)\leq (M+1)(\eta+1)-1.\]
\end{thm}

\begin{proof}
Since the learner only runs one copy of algorithm $A$ at a time, at most $a$ operations are used per round.
\end{proof}

\section{Function class partial ordering}\label{s:ordering}

In this section, we introduce a partial ordering on the class of all families of functions, and we use it to prove some results on mistake-bounded online learning, with and without caps on the number of arithmetic operations. In all learning scenarios considered in this section, we assume that both the learner and the adversary are deterministic.


\begin{definition}
Consider function classes $F$ and $G$ for which every function in $F$ has domain $X_F$ and codomain $Y_F$, with $X_G$ and $Y_G$ defined similarly. Say that $F\precsim G$ (equivalently $G\succsim F$) if there exists a function $\mathcal M:F\to G$, a function $m:X_F\to X_G$, and for each $x\in X_F$ a bijection $s_x:Y_F\to Y_G$ such that for all $f\in F$ and $x\in X_F$, we have $\mathcal M(f)(m(x))=s_x(f(x))$.
\end{definition}

\begin{definition}
If $F$ and $G$ are function classes, say that $F\sim G$ if and only if $F\precsim G$ and $G\precsim F$.
\end{definition}

\begin{theorem}
If $F\precsim G$ are two function classes, then the mistake bound of $F$ is at most that of $G$ in any deterministic learning scenario without caps on the number of arithmetic operations.
\end{theorem}

\begin{proof}
Let $A_G$ be an algorithm that learns $G$ with at most $M$ mistakes. We construct an algorithm that learns $F$ with at most $M$ mistakes. On an input $x$, the learner gives input $m(x)$ to $A_G$, which returns a guess $y$. The learner guesses $s_x^{-1}(y)$. Upon receiving feedback of the form $f(x)=s_x^{-1}(y)$ (such as under standard or strong agnostic feedback), the learner gives $A_G$ the feedback $g(m(x))=y$. Similarly, upon receiving feedback of the form $f(x)\neq s_x^{-1}(y)$ (such as under bandit, delayed ambiguous, or weak agnostic feedback), the learner gives $A_G$ the feedback $g(m(x))\neq y$. Whenever the learner makes a mistake predicting the output of $f$, the algorithm $A_G$ makes a mistake predicting the output of $g=\mathcal M(f)$, since $\mathcal M(f)(m(x))=s_x(f(x))\neq s_x(s_x^{-1}(y))=y$. Therefore, the learner makes at most $M$ mistakes, as desired.
\end{proof}

\begin{cor}
If $F\sim G$ are two function classes, then their mistake bounds are equal to each other in any deterministic learning scenario without caps on the number of arithmetic operations.
\end{cor}

\begin{theorem}
If $F\precsim G$, $\text{opt}_{\text{cap}}(G,a)=M$ in some learning scenario, and the relevant functions $m$ and $s_x^{-1}$ can be computed in at most $a_m$ and $a_s$ arithmetic operations for each individual input, respectively, then $\text{opt}_{\text{cap}}(F,a+a_m+a_s)\leq M$.
\end{theorem}

\begin{proof}
Replicate the above proof, where $A_G$ is an algorithm that learns $G$ with at most $M$ mistakes using at most $a$ operations per round. To run the algorithm that learns $F$, each round the learner only needs to compute one output $m(x)$, run $A_G$ once, and compute one output $s_x^{-1}(y)$. This requires at most $a+a_m+a_s$ operations per round.
\end{proof}

For a prime power $k$ and positive integer $n$, let $F_L(k,n)$ be the function class consisting of all linear transformations $\mathbb F_k^n\to\mathbb F_k$. Let $F_{k,n}$ denote the function class consisting of all functions $\mathbb Z\to\{0,1,\dots,k-1\}$ that output $0$ on all but at most $n$ inputs. It is clear that we have $\optstd(F_L(k,n))=\optstd(F_{k,n})=n$. Long \cite{long} proved that $\optbs(F_L(k,n))=\Omega(nk\log k)$ for all $n\geq 3$. Geneson and Tang \cite{gt24} proved the result for $n=2$ and showed that $\optbs(\cartr(F_{k,n}))\geq k^n-1$ for $r \geq n$. 

Given both of these results, it is natural to investigate whether there exists a function class $F_{k,n}^*$ with $k$ outputs and a standard mistake bound of $O(n)$, such that $\optbs(F_{k,n}^*)=\Omega(nk\log k)$ and $\optbs(\cartr(F_{k,n}^*))=\Omega(k^n)$ for $r \ge n$. By Theorem 5.5, it suffices to find such a function class which satisfies $F_{k,n}^*\succsim F_{k,n}$ and $F_{k,n}^*\succsim F_L(k,n)$. We exhibit such a function class in the next result. 
\begin{theorem}
For prime powers $k$ and integers $r \geq n\geq2$, there exists a function class $F_{k,n}^*$ with $k$ outputs for which \begin{enumerate} \item $\optstd(F_{k,n}^*)\leq2n$, \item $\optbs(F_{k,n}^*)=\Omega(nk\log k)$, and \item $\optbs(\cartr(F_{k,n}^*))\geq k^n-1$. \end{enumerate}
\end{theorem}

\begin{proof}
Let $X=\mathbb F_k^n\sqcup\mathbb Z$. For each $\alpha\in\mathbb F_k^n$ and each function $h\in F_{k,n}$, let $f_{\alpha,h}(x)=\alpha\cdot x$ for $x\in\mathbb F_k^n$ and $f_{\alpha,h}(x)=h(x)$ for $x\in\mathbb Z$. Define the function class $F_{k,n}^*=\{f_{\alpha,h}:\alpha\in\mathbb F_k^n,h\in F_{k,n}\}$.

If $A_1$ and $A_2$ are algorithms that learn $F_L(k,n)$ and $F_{k,n}$ respectively in at most $n$ mistakes in the standard learning scenario, then the following strategy learns $F_{k,n}^*$ in at most $2n$ mistakes: Whenever $x \in \mathbb F_k^n$, the learner makes the guess that $A_1$ would make on input $x$, and gives $A_1$ the feedback received in that round. Whenever $x\in \mathbb Z$, the learner makes the guess that $A_2$ would make on input $x$, and gives $A_2$ the feedback received in that round. The learner makes a mistake if and only if $A_1$ or $A_2$ makes a mistake, and each makes at most $n$ mistakes, so the learner makes at most $2n$ mistakes, as desired.

We now show that $F_{k,n}^*\succsim F_L(k,n)$. Indeed, we let $h$ be an arbitrary function in $F_{k,n}$ and choose $\mathcal M(f_\alpha)=f_{\alpha,h}$ for all $\alpha$, where $f_\alpha$ is the function in $F_L(k,n)$ that maps each $\mathbf x$ to $\alpha\cdot\mathbf x$. Choose $m(x)=x$ for all $x\in\mathbb F_k^n$ and let $s_x$ be the identity function for all $x$. We can verify that $\mathcal M(f)(m(x))=s_x(f(x))$ for all $f\in F_L(k,n)$ and $x\in\mathbb F_k^n$, so $F_{k,n}^*\succsim F_L(k,n)$, as desired.

We now show that $F_{k,n}^*\succsim F_{k,n}$. Similarly to before, we let $f_\alpha$ be an arbitrary function in $F_L(k,n)$ and choose $\mathcal M(h)=f_{\alpha,h}$. Choose $m(x)=x$ for all $x\in\mathbb Z$ and let $s_x$ be the identity function for all $x$. We can verify that $\mathcal M(h)(m(x))=s_x(h(x))$ for all $h\in F_{k,n}$ and $x\in\mathbb Z$, so $F_{k,n}^*\succsim F_{k,n}$, as desired.





\end{proof}

\section{Conclusion}\label{s:conclusion}

There are many open problems and future directions for mistake-bounded online learning with operation caps. We found several families of functions $F$ with $\We(F) = \Le_{F}$, and we showed that any simple online-learnable $F$ must have $\We(F) = \Le_{F}$. It remains to characterize all $F$ with $\We(F) = \Le_{F}$. It also remains to characterize all simple online-learnable $F$. 

\begin{problem}
    Does there exist $F$ with $\optstd(F) < \infty$ for which $\Le_{F} \neq \We(F)$?
\end{problem}

We showed that $\Le_{F_{n, \relu}} \le \A_{F_{n, \relu}} = O(n^3)$ and $\We(F_{n, \relu}) = \Theta(n)$. 

\begin{problem}
    Does $\Le_{F_{n, \relu}} = \We(F_{n, \relu})$?
\end{problem}

In Theorem~\ref{thm_wf_linf} we identified a family $F$ for which $W(F) < \infty$, $\optstd(F) = \infty$, and $\Le_{F} = \infty$. We identified more families in Section~\ref{sec:1layer}.  

\begin{problem}
    What is the least $k$ for which there exists a family $F$ with $\We(F) = k$ such that $\optstd(F) = \infty$?
\end{problem}

In Theorem~\ref{thm:lvsa}, we constructed a family of functions $F$ with $\Le_{F}<\A_{F}$. In Theorem~\ref{thm_la_gap}, we found families of functions $F$ with arbitrarily large gaps between $\Le_{F}$ and $\A_{F}$. Define $\Le_{F,q}$ to be the minimum value of $a$ such that $\optastd(F, a) \le q$ if such an $a$ exists, and otherwise let $\Le_{F,q} = \infty$. By definition, $\Le_{F,q}$ is non-increasing with respect to $q$, where $\Le_{F,q} = \infty$ for $q < \optstd(F)$, $\Le_{F,q} = \A_{F}$ for $q = \optstd(F)$, and $\Le_{F,q} = \Le_{F}$ for $q$ sufficiently large with respect to $F$. For the family $F_r$ in Theorem~\ref{thm_la_gap}, we have $\Le_{F_r,1} = r$ and $\Le_{F_r,2} = 0$ for each $r > 0$. It remains to investigate $\Le_{F,q}$ more generally.

It is easy to construct families of functions $F$ with $\Le_{F} = \A_{F}$. For example, consider any family $F$ of computable functions with $|F| = 1$. In this case, we have $\optstd(F) = 0$, $\Le_{F} = \A_{F} = \We(F)$, and $\Le_{F,q} = \We(F)$ for all $q \ge 0$. For another example, consider any family $F$ with $|F| = 2$. In this case, we have $\optstd(F) = 1$, $\Le_{F} = \A_{F} = \We(F)$, $\Le_{F,0} = \infty$, and $\Le_{F,q} = \We(F)$ for all $q \ge 1$. 

\begin{problem}
    Characterize the families $F$ with $\Le_{F} = \A_{F}$. 
\end{problem}

In Theorem~\ref{ag_strong}, we showed that \[\optagst(F,\eta) = O(\optstd(F) + \eta\ln{k})\] for all families of functions $F$ with codomain of size $k$. In Theorem~\ref{ag_strong_alt}, we proved that \[\optagst(F,\eta)\leq (\optstd(F)+1)(\eta+1)-1\] for all families of functions $F$. Recall that the first upper bound is sharper for $\optstd(F) = \omega(\ln{k})$, while the second upper bound is sharper for $\optstd(F) = o(\ln{k})$. For all $F$, \[\optagst(F,\eta) = \Omega(\optstd(F)+\eta),\] but it remains to determine the maximum possible value of $\optagst(F,\eta)$ with respect to $\optstd(F)$ and $\eta$.

As for agnostic learning with weak reinforcement, we proved in Theorem~\ref{ag_weak} that \[\optagwe(F,\eta) = O((k \ln{k})\optstd(F) + k \eta)\] for all families of functions $F$ with codomain of size $k$. This resolves an open problem from the papers \cite{filmus,gt24}, which independently proved that there exist families of functions $F$ with codomain of size $k$ such that \[\optagwe(F,\eta) = \Omega((k \ln{k})\optstd(F) + k \eta).\] It still remains to determine the best possible constants in the bounds. 

In several results, we bounded the worst-case errors for various online learning scenarios with restrictions on the number of arithmetic operations. It would be interesting to sharpen and extend these inequalities, i.e., to determine what is the sharpest possible upper bound on $\text{opt}_{\text{cap},\text{scenario}_\text{1}}(F, a_1)$ with respect to $\text{opt}_{\text{cap},\text{scenario}_\text{2}}(F, a_2)$ for any mistake-bounded online learning scenarios $\text{scenario}_\text{1}, \text{scenario}_\text{2}$ and nonnegative integers $a_1, a_2$.

There are connections between online learning and several other learning settings including autoregressive learning \cite{jvb25}, PAC learning \cite{littlestone0}, private learning \cite{abl22}, and transductive learning \cite{hms23, hrss24}. As with online learning, it would be interesting to investigate the effects of operations caps in these other settings. 

The idea of mistake-bounded online learning with operation caps was based on the 2023 AI executive order \cite{exec_order} which required that cloud service providers must report customers that use over $10^{26}$ floating point operations in a single run. In addition to the motivation from the AI executive order, we can also view $\Le_{F}$ and $\A_{F}$ as measures of complexity for the family $F$. In particular, $\Le_{F}$ is a measure of complexity in the standard scenario for learning the family $F$ with a finite mistake bound, while $\A_{F}$ is a measure of complexity in the standard scenario for learning the family $F$ with as few mistakes as possible. 

\section*{Acknowledgements}
JG was supported by the Woodward Fund for Applied Mathematics at San Jose State University, a gift from the estate of Mrs. Marie Woodward in memory of her son, Henry Tynham Woodward. He was an alumnus of the Mathematics Department at San Jose State University and worked with research groups at NASA Ames. JG thanks Alvin Moon, Padmaja Vedula, and Simon Bar-on for helpful discussions about the executive order \cite{exec_order}.

\end{document}